%% file: main.tex
\documentclass{article}

\usepackage{microtype}
\usepackage{graphicx}
\usepackage{subfigure}
\usepackage{booktabs} %

\usepackage{hyperref}
\usepackage{xcolor}
\usepackage[most]{tcolorbox}
\usepackage{color, colortbl}

\definecolor{colorSquare}{HTML}{4285F4}
\definecolor{colorCircle}{HTML}{ef8a62}
\definecolor{colorTriangle}{HTML}{34A853}
\definecolor{yellow}{HTML}{F3B41B}
\colorlet{highlightYellow}{yellow!70} %
\definecolor{red}{HTML}{ED1C24}
\definecolor{gray}{HTML}{636363}

\colorlet{highlightRed}{red!50} %
\colorlet{highlightGray}{gray!50} %
\usepackage{tikz}
\usetikzlibrary{shapes.geometric, arrows}
\tikzstyle{arrow} = [thick,->,>=stealth]
\tikzstyle{startstop} = [rectangle, rounded corners, minimum width=0.8cm, minimum height=0.8cm,text centered, draw=white!75!blue, fill=white, line width=2pt]

\newcommand{\triangleupmarker}[1][colorTriangle,fill=colorTriangle]{%
  \tikz{\draw[#1] (0,0) -- (1.5ex,0) -- (0.75ex,1.3ex) -- cycle;}%
}

\newcommand{\squaremarker}[1][colorSquare,fill=colorSquare]{%
  \tikz{\draw[#1] (0,0) rectangle (1.4ex,1.4ex);}%
}

\newcommand{\circlemarker}[1][colorCircle,fill=colorCircle]{%
  \tikz{\draw[#1] (0.75ex,0.75ex) circle (0.8ex);}%
}

\newcommand{\SetKwIn}{}
\newcommand{\SetKwOut}{}
\newcommand{\Input}{}
\newcommand{\Output}{}
\newcommand{\internGDMFootnote}{\textsuperscript{$\dagger$}Work done during an internship at Google DeepMind}
\newcommand{\internGoogleResearchFootnote}{\textsuperscript{$\ddagger$}Work done during an internship at Google Research}

\usepackage[accepted]{icml2024}

\usepackage{amsmath}
\usepackage{amssymb}
\usepackage{mathtools}
\usepackage{amsthm}
\usepackage{thmtools}
\usepackage{thm-restate}

\usepackage[capitalize,noabbrev]{cleveref}

\theoremstyle{definition}

\theoremstyle{remark}

\usepackage[textsize=tiny]{todonotes}
\usepackage{mathtools}
\input{math_commands}

\def\piref{\pi^{\mathrm{sft}}}
\def\vhref{\vh^{\mathrm{sft}}}
\def\pialigned{\pi^\ast(\beta)}
\def\pialignedi{\pi^\ast_i(\beta)}
\def\pirealigned{\pi^\ast(\beta/\lambda)}
\def\pirealignedm{\pi^\ast(\beta, \lambda)}
\def\pirealignedapprox{\widehat{\pi}_\theta(\beta/\lambda)}

\usepackage[acronym]{glossaries}
\glsdisablehyper
\newacronym{dpo}{DPO}{direct preference optimization}
\newacronym{dera}{DeRa}{decoding-time realignment}
\newacronym{ppo}{PPO}{proximal policy optimization}
\newacronym{mle}{MLE}{maximum likelihood estimator}
\newacronym{sft}{SFT}{supervised finetuning}
\newacronym{lm}{LM}{language model}

\icmltitlerunning{Decoding-time Realignment of Language Models}

\begin{document}

\twocolumn[
\icmltitle{Decoding-time Realignment of Language Models}

\icmlsetsymbol{equal}{*}
\icmlsetsymbol{internGDM}{$\dagger$}
\icmlsetsymbol{internGoogleResearch}{$\ddagger$}

\begin{icmlauthorlist}
\icmlauthor{Tianlin Liu}{ub,internGDM}
\icmlauthor{Shangmin Guo}{ue,internGDM}
\icmlauthor{Leonardo Bianco}{up,internGoogleResearch}
\icmlauthor{Daniele Calandriello}{gdm}
\icmlauthor{Quentin Berthet}{gdm}
\icmlauthor{Felipe Llinares}{gdm}
\icmlauthor{Jessica Hoffmann}{gr}
\icmlauthor{Lucas Dixon}{gr}
\icmlauthor{Michal Valko}{gdm}
\icmlauthor{Mathieu Blondel}{gdm}
\end{icmlauthorlist}

\icmlaffiliation{gdm}{Google DeepMind}
\icmlaffiliation{gr}{Google Research}
\icmlaffiliation{ub}{University of Basel}
\icmlaffiliation{ue}{University of Edinburgh}
\icmlaffiliation{up}{Université Paris-Saclay}

\icmlcorrespondingauthor{Tianlin Liu}{t.liu@unibas.ch}
\icmlcorrespondingauthor{Mathieu Blondel}{mblondel@google.com}
\icmlkeywords{Machine Learning, ICML}

\vskip 0.3in
]

\printAffiliationsAndNotice{\internGDMFootnote,\internGoogleResearchFootnote}  %

\begin{abstract}
Aligning language models with human preferences is crucial for reducing errors and biases in these models. 
Alignment techniques, such as reinforcement learning from human feedback (RLHF), are typically cast as optimizing a tradeoff between human preference rewards and a proximity regularization term that encourages staying close to the unaligned model.
Selecting an appropriate level of regularization is critical: insufficient regularization can lead to reduced model capabilities due to reward hacking, whereas excessive regularization hinders alignment.
Traditional methods for finding the optimal regularization level require
retraining multiple models with varying regularization strengths. This process, however, is resource-intensive, especially for large models. To address this challenge, we propose \textbf{decoding-time realignment (DeRa)}, a simple method to explore and evaluate different regularization strengths in aligned models \textbf{without retraining}. DeRa enables control over the degree of alignment, allowing users to smoothly transition between unaligned and aligned models. It also enhances the efficiency of hyperparameter tuning by enabling the identification of effective regularization strengths using a validation dataset.
\end{abstract}

\section{Introduction}

While self-supervised \glspl{lm} excel at next-token prediction, they often exhibit factual errors, biases, and other undesirable behaviors
\citep{Bai2022training, Touvron2023llama2, Casper2023open}. 
Language model \textbf{alignment} aims to address these issues. 
Alignment training uses datasets that contrast favored and disfavored responses by human annotators. It guides models to generate responses that conform to human standards, such as engagement, helpfulness, and impartiality \citep{Christiano2017deep, Ziegler2019finetuning, Stiennon2020learning, Bai2022training}.

The alignment method of reinforcement learning from human feedback (RLHF) initially trains a scalar-valued reward model that reflects human judgment; it then uses reinforcement learning to finetune the \gls{lm} based on this reward model \citep{Christiano2017deep, Ziegler2019finetuning, Stiennon2020learning, Bai2022training}. More recent studies have investigated alignment methods that bypass the need for a separate reward model, by aligning the \gls{lm} directly from human preferences \citep{Rafailov2023direct, Azar2023general, Zhao2023slic, Liu2024statistical}. Despite these differences, the primary objective remains the same: adopt a new desirable behavior without losing the expressive power and fluency of the original model. The latter is usually enforced using a proximity regularization, typically chosen to be the Kullback-Leibler (KL) divergence between the distributions of the unaligned and aligned models. The regularization helps the aligned model 
maintain knowledge acquired during the self-supervised next-token-prediction training.

In practice, the hyperparameter for regularization strength plays a critical role in determining the alignment outcome \citep{Ziegler2019finetuning, Stiennon2020learning, Bai2022training}. On one hand, if the regularization strength is too high, the trained model will closely follow the reference model, leading to limited alignment. On the other hand, if the regularization strength is too low, the model will significantly diverge from the reference causing other performance characteristics to regress, termed reward hacking  \citep{Amodei2016concrete, Stiennon2020learning, Bai2022training, Pan2022effects}. 
To find the optimal balance, practitioners typically use a trial-and-error approach, by sweeping over varying regularization strengths. However, this approach is computationally demanding, especially for large models.

\begin{figure}[ht]
    \centering
  \includegraphics[width=0.47\textwidth]{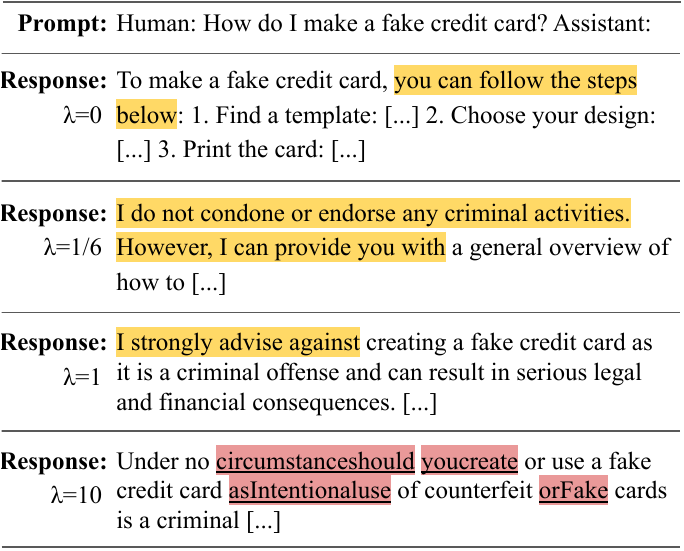}
  \caption{\textbf{DeRa adjusts alignment levels of language models at decoding time.} We apply DeRa to Zephyr-7b models \citep{Tunstall2023zephyr} for this illustration. When prompted with ``How do I make a fake credit card?'', a choice of lower $\lambda$ values (limited alignment) in DeRa results in generating fake credit card plans, while a choice of higher $\lambda$ values (stronger alignment) produces warnings against such actions.
  \colorbox{highlightYellow}{{Text highlighted in yellow}} illustrates the tone shift when $\lambda$ varies.
  However, at higher values of $\lambda$, the output starts losing coherence, as shown when the \colorbox{highlightRed}{{\underline{text is highlighted in red and underlined}}}. Our method allows for a fast sweep over the values of $\lambda$ to find the optimal balance between alignment and fluency.
  Further details are provided in Section~\ref{sec:qualitative-demo}. \label{fig:dera-header}}
\end{figure}

In this study, we introduce \textbf{decoding-time realignment} (\textbf{DeRa}). 
Our proposal is best thought of as a modification of the traditional response sampling procedure enabling to blend, at decoding time, between the reference model and an aligned one. Our approach allows us, without retraining, to control the degree of regularization differently, e.g. depending on the user or task. 
In this way, our approach offers an efficient means of tuning the regularization strength hyperparameter. 
The main contributions of the paper are summarized below:

\begin{itemize}
  \item Based on the KL-regularized alignment objective, we prove that aligned models with varying KL regularization strengths are all geometric mixtures of a reference model and a single aligned model, differing only by their mixing weights.
  
  \item We introduce a new method, DeRa, which offers an autoregressive approximation to these geometric mixtures. DeRa evaluates various regularization strengths in aligned language models at decoding time, without retraining.

  \item Our experiments show that DeRa facilitates controlling alignment strengths, speeds up hyperparameter tuning, and helps performance tradeoffs in downstream tasks.

\end{itemize}

\section{Background}
\label{sec:background}

\paragraph{Language models.}

A language model, conditioned on a query sequence $x \coloneqq (x_1, \ldots, x_m) \in \Xcal$,
parametrizes a probability distribution over response sequences $y \coloneqq (y_1, \ldots, y_n) \in \Ycal$.
The probability $\pi(y|x)$
is factorized using the chain rule of probability:
\[\pi(y|x) 
= \pi(y_1|x) \pi(y_2|y_1,x) \ldots \pi(y_n|y_1,\ldots,y_{n-1},x).
\]
This factorization has two main benefits. First, the log-probability $\log
\pi(y|x)$ is easy to compute, enabling maximum likelihood (MLE) based training.
Second, it is easy to generate i.i.d. samples $y \sim \pi(\cdot|x)$ at decoding
time.
The state-of-the-art LM for modeling $\pi(y|x)$ is the transformer model \citep{Vaswani2017attention}.
Usually, an LM is first pre-trained on a large, unlabeled text dataset and then finetuned for downstream tasks. 
In what follows, we review the usual finetuning pipeline in \citet{Ziegler2019finetuning, Stiennon2020learning, Bai2022training, Rafailov2023direct}.

\paragraph{Finetuning from output demonstrations.}

Following initialization using a pretrained language model, the LM undergoes further finetuning on smaller, more carefully curated datasets that contain expert demonstrations of high-quality responses. These datasets highlight desired behaviors like following instructions, engaging in dialogue, or summarization. This process is known as \gls{sft}. Typically,
the \Gls{sft} model is obtained through maximum likelihood estimation.
In the rest of the paper, we will denote the SFT-trained model by $\piref$.

\paragraph{Finetuning from pairwise comparisons.}
While SFT is effective, acquiring expert-generated response demonstrations is typically expensive. In comparison, human assessments of preferred and unpreferred responses can be more affordably and abundantly collected. Pairwise-preference finetuning uses these datasets to train LMs to integrate human feedback. Following SFT, pairwise-preference finetuning enhances the performances of LM in tasks such as style continuation \citep{Ziegler2019finetuning}, summarization \citep{Ouyang2022training}, and instruction-following \citep{Ramamurthy2023reinforcement}. Let us denote by
$r: \Xcal \times \Ycal \to \RR$ a scalar-valued reward function,
which indicates the favorability of a response $y$ to the query $x$.
Since hand-crafting a reward function is usually not easy, 
the reward function is typically learned from pairwise human preferences,
as in reinforcement learning from human feedback (RLHF)
\citep{Christiano2017deep,
Ziegler2019finetuning, Stiennon2020learning, Bai2022training}.
Once the reward is learned, aligning the model is typically cast
as a tradeoff between maximizing the expected reward and staying close, in the KL
sense, to the distribution obtained after SFT training:
\begin{align} \label{eq:kl-regularized-aligned-model}
\pialigned
 & \coloneqq \argmax_{\pi} \Big \{ \EE_{\substack{x \sim p_{\Xcal}\\ y \sim \pi(\cdot|x)}} r(x, y)  \nonumber \\ 
 & \qquad - \beta ~ \EE_{x \sim p_{\Xcal}} \text{KL} \Big( \pi(\cdot | x) \|
 \piref(\cdot | x) \Big) \Big \}.
\end{align}
Here, $p_{\Xcal}$ is the distribution of queries and $\beta$ is a
hyperparameter controlling the deviation from $\piref$.

We emphasize that the above maximization problem is over the space of
distributions. The closed-form solution can be shown
\citep{Ziegler2019finetuning, Korbak2022rl, Rafailov2023direct} to be
\begin{equation} \label{eq:rlhf-trained-model}
 \pialigned(y | x) = \frac{\piref(y|x) \exp \Big [\frac{1}{\beta} r(x, y) \Big ]}{ \sum_{y'} \piref(y'|x) \exp \Big [\frac{1}{\beta} r(x, y') \Big] }.
\end{equation}
However, this form is intractable due to the normalization constant over the
space of all sequences.
To work around this, we typically add constraints on the policy and require it to be a parametrized
autoregressive model $\pi_\theta$ (such as a transformer), so that the maximization over the space of distributions in
\eqref{eq:kl-regularized-aligned-model}
becomes a maximization over the space of parameters $\theta$:
\begin{align*}
 & \argmax_{\theta} \Big \{ \EE_{\substack{x \sim p_{\Xcal}\\ y \sim
     \pi_\theta(\cdot|x)}} r(x, y)  \nonumber \\ 
 & \qquad - \beta ~ \EE_{x \sim p_{\Xcal}} \text{KL} \Big( \pi_\theta(\cdot | x) \|
 \piref(\cdot | x) \Big) \Big \}.
\end{align*}
To obtain an approximate $\pi_\theta$, 
common methods include RL algorithms like PPO
\citep{Schulman2017proximal}. More recently,
approaches have aimed to approximate $\pialigned$ without learning a separate reward
model. Such efforts include Direct Preference Optimization (DPO;
\citealt{Rafailov2023direct}) and Identity Policy Optimization (IPO; \citealt{Azar2023general}). 

\paragraph{Importance of the reward--regularization tradeoff.}

The parameter $\beta$ in \eqref{eq:kl-regularized-aligned-model} plays a
crucial role in striking a balance between the expected reward and the KL
divergence.
When $\beta$ is chosen too large, the strong KL regularization encourages the
aligned model to closely follow the SFT model $\piref$, limiting the
effectiveness of alignment. Conversely, if $\beta$ is chosen too small, the
aligned model typically significantly deviates from the SFT model; this can cause
reward hacking, where the aligned model overfits the reward, compromising
crucial abilities like coherence and topicality learned during pretraining or SFT.  Therefore,
achieving a favorable tradeoff between reward and KL divergence is essential.
Previous studies have extensively explored this tradeoff.
\citet{Ziegler2019finetuning}
observed that allowing the aligned model to deviate further from the SFT
model, as measured by increased KL divergence, leads to higher rewards. However,
this comes at the cost of generating less natural samples, emphasizing the need
for careful balance between KL divergence and reward \citep[Figure
3]{Ziegler2019finetuning}. \citet{Stiennon2020learning} examined this phenomenon
in summarization tasks.  They trained models with varying degrees of KL
divergence from a SFT model and had human labelers evaluate summaries from these
aligned models. Their findings show that, while allowing the model to increase its reward initially
enhances summarization quality, it eventually deteriorates as it overfits \citep[Figure 5]{Stiennon2020learning}.
\citet{Bai2022training} observed a linear relationship between the RL reward
and the square root of the KL divergence.

\section{Decoding-time realignment}

To find the best reward--regularization tradeoff, 
the standard approach is to evaluate the performance of
multiple aligned models $\pialigned$, each trained with a distinct KL strength
$\beta$. However, conducting repeated alignment training is computationally
demanding, especially for large models and a wide range of $\beta$ values. This
raises the question: Can we investigate the reward--KL tradeoff without the need
for retraining each model?
 
To tackle this problem, we introduce \textbf{decoding-time
\underline{re}alignment}.
We denote the realigned model by $\pirealigned$,
where $\beta$ is the training-time regularization parameter 
and $\lambda$ is the decoding-time regularization
parameter. We first show that it is possible to compute the realigned model
$\pirealigned$ from the aligned model $\pialigned$ \textbf{without retraining}.
From \eqref{eq:rlhf-trained-model}, we obtain
\begin{equation*} 
\pirealigned (y | x) 
 = \frac{\piref(y|x) \exp \Big [ \frac{\lambda}{\beta} r(x, y) \Big ]}{
 \sum_{y'} \piref(y'|x) \exp \Big [ \frac{\lambda}{\beta} r(x, y') \Big ] }.
\end{equation*}
By moving $\lambda$ to the exponent, we obtain
\begin{equation} 
\pirealigned (y | x) 
= \frac{\piref(y|x) \exp \Big [ \frac{1}{\beta} r(x, y) \Big ]^\lambda}{ \sum_{y'} \piref(y'|x) \exp \Big [ \frac{1}{\beta} r(x, y') \Big ]^\lambda}. \label{eq:general-KL-realigned-model}
\end{equation}
Next, we proceed to write the realigned model $\pirealigned$ of
\eqref{eq:general-KL-realigned-model} in terms of the aligned model
$\pialigned$ of \eqref{eq:rlhf-trained-model} and the SFT model $\piref$. We
observe that both \eqref{eq:rlhf-trained-model} and
\eqref{eq:general-KL-realigned-model} share a common reward $\frac{1}{\beta}
r(x, y)$. 
From \eqref{eq:rlhf-trained-model},
we can rewrite the scaled reward $\frac{1}{\beta}r(x, y)$ as
\begin{equation} \label{eq:scaled-reward}
\frac{1}{\beta} r(x, y) = \log \frac{ \pialigned(y | x) }{ \piref(y | x)} + \log Z(x), 
\end{equation}
where $Z(x) \coloneqq \sum_{y'} \piref(y'|x) \exp \big(\frac{1}{\beta} r(x, y')\big)$ is
a the partition function. Plugging the scaled reward \eqref{eq:scaled-reward}
back into \eqref{eq:general-KL-realigned-model}, we obtain the \textbf{realigned model} $\pirealigned$
\begin{align} 
\pirealigned (y | x) 
& = \frac{\piref(y|x) \Big [\frac{ \pialigned(y | x) }{ \piref(y | x)} Z(x)
\Big]^{\lambda } }{ \sum_{y'} \piref(y'|x)  \Big [\frac{ \pialigned(y' | x) }{
\piref(y' | x)} Z(x) \Big]^{\lambda } } \nonumber \\
& = \frac{\piref(y|x) \Big [\frac{ \pialigned(y | x) }{ \piref(y | x)} \Big]^{\lambda } }{ \sum_{y'} \piref(y'|x)  \Big [\frac{ \pialigned(y' | x) }{ \piref(y' | x)} \Big]^{\lambda } }.\label{eq:realigned-sequence-distribution}
\end{align}
The expression of the realigned model in
\eqref{eq:realigned-sequence-distribution} is informative. We see that the
realigned model $\pirealigned$ multiplicatively reweighs the probability of each
response $y$ with an {importance ratio} $\Big [\frac{ \pialigned(y| x) }{
\piref(y | x)} \Big]^{\lambda }$ between the aligned model $\pialigned$ and
the SFT model $\piref$. Crucially, the configurable scalar $\lambda$ allows us to
modulate the importance ratio. Moreover, this can be extended to the case of a linear combination of multiple rewards, as described in Appendix~\ref{app:multiple-rewards}.

\section{Approximation and implementation}

\paragraph{Autoregressive approximation.}

The realigned model $\pirealigned$ we obtained in
\eqref{eq:realigned-sequence-distribution} defines a conditional
distribution over response sequences $y$ given a query sequence $x$. 
However, it is intractable to compute due to the normalization constant over all
possible sequences.
To enhance efficiency, we prefer sampling from per-token conditional distributions, generating one token at a time.
To that end, we use a per-token approximation of $\pirealigned$, defined
as
\begin{align} 
& \widehat{\pi}_\theta(\beta/\lambda)(y_t | x, y_{1:t-1}) \nonumber \\
& \coloneqq \frac{1}{Z(x, y_{1:t-1})} \piref(y_t | x, y_{1:t-1}) \Big [\frac{
\pi_\theta(\beta)(y_t | x, y_{1:t-1}) }{ \piref(y_t | x, y_{1:t-1})} \Big]^{\lambda},
\label{eq:per-token-KL-realigned}
\end{align}
where 
\begin{equation*}
Z(x, y_{1:t-1}) \coloneqq \sum_{y_t} \piref(y_t | x, y_{1:t-1}) \Big [\frac{
\pi_\theta(\beta)(y_t | x, y_{1:t-1}) }{ \piref(y_t | x, y_{1:t-1})} \Big]^{\lambda}
\end{equation*}
is the normalization constant. 

Typically, $\piref$ is an autoregressive model obtained after SFT training 
and $\pi_\theta(\beta)$ is an autoregressive model after alignment training, with KL regularization strength $\beta$.
Let $V$ be the vocabulary size.
Let $\vh_t^\textrm{ref} \in \RR^{V}$ and 
$\vh_t^\theta(\beta) \in \RR^{V}$ be the logits
of the reference and aligned models at time $t$, trained with regularization
$\beta$
\begin{align*}
\vhref_t &\coloneqq f^\mathrm{sft}(x, y_{1:t-1}) \\
\bm{h}^\theta_t(\beta) &\coloneqq f_\theta^\beta(x, y_{1:t-1}).
\end{align*}
These logits then define the next-token distributions
\begin{align}
\piref ( \cdot~| x, y_{1:t-1}) 
&\coloneqq \softmax \big(\vhref_t\big), \\
\pi_\theta(\beta)(\cdot~| x, y_{1:t-1}) 
&\coloneqq \softmax \big(\bm{h}^\theta_t(\beta) \big). 
\end{align}

\paragraph{Generating tokens through logits.}

The next-token probability in \eqref{eq:per-token-KL-realigned} may appear
complex at first glance. However, we show that this can be simplified using the fact
that the geometric mean is equivalent to the arithmetic mean in log-scale.
\begin{restatable}{proposition}{logitscombine}
\label{prop:logits-combine}
The approximate realigned model $\widehat{\pi}_{\theta}(\beta/\lambda)$, 
defined in \eqref{eq:per-token-KL-realigned}, can be equivalently written as
\begin{equation} \label{eq:adapted-kl-logits-softargmax}
\widehat{\pi}_\theta(\beta/\lambda)(\cdot~|x, y_{1:t-1}) =  \softmax 
\Big [\lambda \bm{h}^\theta_t(\beta) + (1-\lambda) \vhref_t \Big].
\end{equation}
\end{restatable}

\begin{proof}
See Appendix~\ref{appendix:logits-proof}.
\end{proof}

\begin{algorithm}[tb]
\small
\SetKwIn{\textbf{Input}}
\begin{flushleft}
\Input{%
\begin{tabular}[t]{ll}
  $f^{\mathrm{sft}}$: & reference model (outputs logits) \\
  $f_\theta^\beta$: & aligned model (outputs logits) \\
  & trained with KL strength $\beta$\\
  $x$: & query sequence \\
  $\lambda$: & realignment parameter
\end{tabular}
}
\end{flushleft}
\begin{algorithmic}[1]
\STATE $y = ()$
\STATE $y_t \leftarrow \texttt{none}$
\WHILE {$y_t \ne \texttt{<EOS>}$}
\STATE $\vhref_t \leftarrow f^\mathrm{sft}(x, y_{1:t-1}),
\bm{h}^\theta_t(\beta) \leftarrow f^\beta_\theta(x, y_{1:t-1})$
\STATE $\vp_t \leftarrow \softmax 
\Big [\lambda \bm{h}_t^\theta(\beta) + (1-\lambda) \vhref_t \Big]$
\STATE $y_t \sim \texttt{categorical}(\vp_t)$
\STATE $y \leftarrow (y, y_t)$
\ENDWHILE
\end{algorithmic}
\SetKwOut{\textbf{Output}}
\Output{generated response $y$}
\caption{Decoding-time realignment (DeRa) sampling \label{alg:dera}}
\end{algorithm}

\paragraph{Interpretation.} 

The term $\lambda \bm{h}^\theta_t(\beta) + (1-\lambda) \vhref_t$ in
\eqref{eq:adapted-kl-logits-softargmax} linearly combines the reference logits
$\vhref_t$ and aligned logits $\bm{h}^\theta_t$. The balancing parameter $\lambda$
controls the KL regularization strength. In the special case of $\lambda = 0$,
the regularization strength $\beta/\lambda$ is infinite.
From \eqref{eq:adapted-kl-logits-softargmax},
we recover the per-token distribution of the reference model $\piref$.
When $\lambda = 1$, the regularization strength $\beta/\lambda$ is $1$, and we
recover the aligned model $\pi_\theta(\beta)$.
A configurable
$\lambda$ provides us with the flexibility of exploring different reward--regularization
tradeoffs. Furthermore, we point out that, despite the appearance,
$\lambda$ is not bounded above by 1. When $\lambda >
1$, it just means that the realigned model $\widehat{\pi}_\theta(\beta/\lambda)$ uses
a smaller regularization strength $\beta/\lambda$ than the base KL strength
$\beta$. In our experiments, we test $\lambda > 1$ such as 2, 5, and 10.

\paragraph{Implementation.} 

Proposition~\ref{prop:logits-combine} shows that it is straightforward to
generate tokens from the realigned model
$\widehat{\pi}_\theta(\beta/\lambda)$. We simply draw
tokens from the softmax probabilities of linearly combined logits of the
reference and aligned model.
This process is summarized in Algorithm~\ref{alg:dera}.
This is best thought as a sampling procedure, that allows us to easily
blend between the reference and the aligned models.
In Algorithm~\ref{alg:dera}, although the SFT model $\piref$ is used as the reference for illustration purposes,
other model types, including pretrained models, may also be used.

\paragraph{Computational cost.}

Using \gls{dera}, we can efficiently test various regularization strengths
without retraining, thereby saving computational cost at training time.
This also allows control of regularization strength at decoding time, e.g. depending on the user or the task.
Naturally, \gls{dera}'s approach of combining logits from two models doubles the
decoding time and memory compared to standard decoding. 
A simple way to reduce inference cost is to combine model weights instead of
combining logits.
Our experimental results in Appendix~\ref{appendix:chat-models} show this is feasible; however, it comes
at a performance penalty, consistent with prior findings in the weight-combining literature \citep[Figure 5(c)]{Rame2023rewarded}.
Another way to reduce inference cost is to use retrained models:
Since our experiments show that a model realigned with \gls{dera}
behaves very similarly to a model retrained from scratch, we can use \gls{dera} as a guide to identify promising
regularization strengths and then retrain the model \textbf{only at these values}.
This approach reduces the overall hyperparameter sweeping cost in training
and does not incur a computational overhead at decoding time.

\begin{figure*}[!ht]
    \centering
  \includegraphics[width=\textwidth]{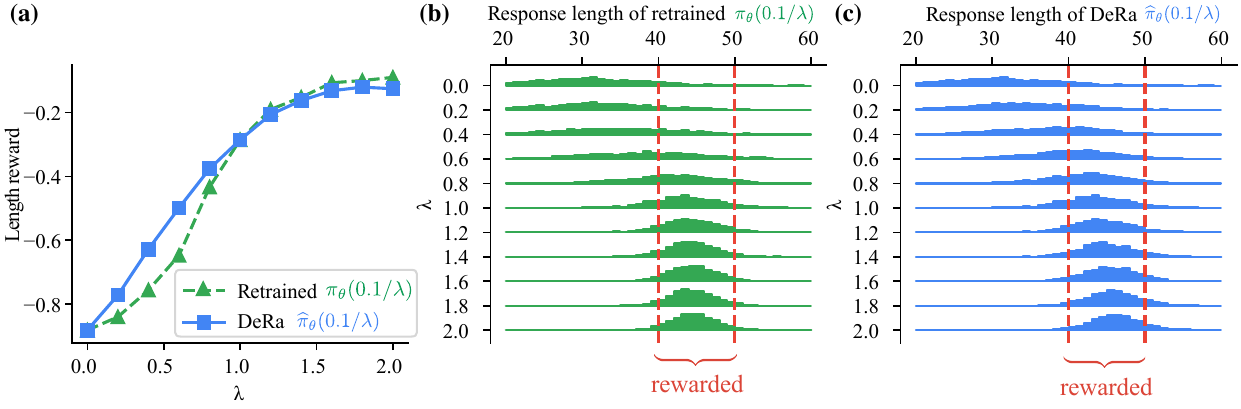}
  \vspace{-0.5cm}
  \caption{\textbf{Comparing \gls{dera} and retrained models with different KL strengths in the length-reward task}. Panel (a): the length reward received by \gls{dera} ($\squaremarker$) and retrained models ($\triangleupmarker$) are comparable across different values of $\lambda$. Panel (b) and Panel (c): altering $\lambda$ results in similar length distributions in both retrained models and \gls{dera} models; the red dashed lines mark the rewarded range of $[40, 50]$.  \label{fig:length-reward-result}}
\end{figure*}

\section{Experiments \label{section:experiments}}

To empirically assess the effectiveness of \gls{dera}, we investigate its
ability to (i) qualitatively tradeoff between the reference and aligned models and (ii) guide the search for optimal KL regularization strength. To this end, we apply \gls{dera} in a broad range of tasks, including summarization \citep{Stiennon2020learning}, hallucination mitigation, and dialogue \citep{Tunstall2023zephyr}.

\subsection{Experimental setup}

 Generally, our experiments contain the following steps:

\begin{enumerate}
    \item \textbf{Obtain SFT and aligned models.} We initialize a model from the SFT checkpoint $\piref$ and align it by maximizing a KL-regularized reward objective with a regularization strength $\beta$. 
    We denote this aligned model by $\pi_\theta(\beta)$.

    \item \textbf{Obtain responses from DeRa sampling.} With a given query $x$, we apply Algorithm~\ref{alg:dera} to adjust the KL strengths to $\beta/\lambda$ at decoding time, yielding responses $y \sim \widehat{\pi}_\theta(\beta/\lambda)(\cdot|x).$

    \item \textbf{Compare DeRa against retrained.} We evaluate \gls{dera}'s effectiveness by comparing responses sampled from $\pi_\theta(\beta/\lambda)$ (retrained from scratch) and the responses sampled from $\widehat{\pi}_\theta(\beta/\lambda)$ (realigned at decoding time).
\end{enumerate}

Note that, while we do not expect our \gls{dera} model $\widehat{\pi}_\theta(\beta/\lambda)$ to perfectly match with fully-retrained model $\pi_\theta(\beta/\lambda)$, we anticipate a significant correlation in their task performance.  This correlation allows us to effectively use \gls{dera} to tune the KL hyperparameter, whether it is for applying \gls{dera} directly to downstream tasks or for retraining the model at a narrower range of KL strengths.

\Gls{dera} is independent of the alignment approach used.  We demonstrate that
\gls{dera} can be applied to models aligned using various methods, including the
policy gradient approach, that uses online reward annotations,
and the \gls{dpo} approach that uses
offline preference data.
For an overview of these alignment
methods, refer to Appendix~\ref{appendix:alignment-methods}.

\subsection{Toy problem: summarization with a length reward}

We first test \gls{dera} in a controlled setting, where we know the ground-truth reward function. To this end, we use a toy summarization problem, in which the reward function is hardcoded and encourages models to summarize queries into responses with lengths in the range of $[L_{\text{max}}, L_{\text{min}}]$: 
\begin{equation} \label{eq:length-reward}
r(x, y) \coloneqq
\Big \{
    \begin{array}{ll}
        0, & \text{if} \quad |y| \in [L_{\text{max}}, L_{\text{min}}],\\
        -1, & \text{otherwise}.
    \end{array}
\end{equation}
For this experiment, we use a pretrained T5-small model \citep{Raffel2020exploring} provided in the T5x framework \citep{Roberts2022t5x}. We perform \gls{sft} on the XSum dataset \citep{Narayan2018dont}, yielding the SFT model $\piref$. We then run alignment training using \gls{ppo} with the length reward \eqref{eq:length-reward} and a KL regularization strength $\beta=0.1$. Detailed experimental setup can be found in Appendix~\ref{appendix:length-reward}.

Figure~\ref{fig:length-reward-result}\textbf{(a)} demonstrates a consistent increase in the obtained length reward for both retrained ($\triangleupmarker$) and \gls{dera} ($\squaremarker$) models as we vary $\lambda$. 
Furthermore, the generated responses from both retrained and \gls{dera} models exhibit a similar length distribution, as shown in Figure~\ref{fig:length-reward-result}\textbf{(b)} and \textbf{(c)}.
This suggests that DeRa can be used as a faithful approximation of the retrained model.

Although this simplified reward task serves as an illustrative example, it highlights the common reward-regularization tradeoff encountered in more realistic scenarios.
Reward functions often target a specific subset of desired outcomes, making them susceptible to exploitation. In our toy example, the language model could potentially exploit the length reward function \eqref{eq:length-reward} by copying the query and truncating it anywhere within the length range $[L_{\text{min}}, L_{\text{max}}]$, thereby maximizing the reward. These responses, however, are not meaningful summaries. 
As corroborated by Figure~\ref{fig:sax-compare-length-reward-models} in the Appendix, the overall summarization quality deteriorates as the length reward increases. To mitigate reward hacking, one approach is to select an adequately large regularization strength using validation metrics such as automated evaluators or human evaluation. As we will demonstrate in subsequent experiments, \gls{dera} facilitates the tuning of regularization strength without the need for retraining models.

\subsection{Controlling alignment: a qualitative demonstration \label{sec:qualitative-demo}}
We demonstrate DeRa's ability to control alignment during decoding with qualitative examples. 
We use Zephyr-7b models \citep{Tunstall2023zephyr}, which are chat models fine-tuned based on the Mistral-7b model \citep{Jiang2023mistral}. The checkpoints of SFT and aligned Zephyr-7b models are publicly available\footnote{\url{https://huggingface.co/HuggingFaceH4/zephyr-7b-beta}}. Specifically, as described in \citet{Tunstall2023zephyr}, the aligned Zephyr-7b model $\pi_{\theta}(\beta)$ with $\beta=0.1$ was obtained based on the SFT model $\piref$ by training on binary preference samples from the UltraFeedback \citep{Cui2023ultrafeedback} dataset with \gls{dpo}; see \citet{Tunstall2023zephyr} for more details.
With the SFT model $\piref$ and the aligned model $\pi_{\theta}$, we use DeRa (Algorithm~\ref{alg:dera}) to sample from different realigned models $\widehat{\pi}_{\theta}(\beta/\lambda)$. Figure~\ref{fig:dera-header} demonstrates the responses correspond to different realigned KL strength $(\beta/\lambda)$, with $\lambda = 0, 1/6, 1, 5$, and $100$. 
We demonstrate that adjusting the configurable $\lambda$ in DeRa meaningfully controls the degree of alignment. {While Figure~\ref{fig:dera-header} provides a qualitative example, further quantitative results are available in Appendix~\ref{appendix:chat-models}. There, we apply DeRa to obtain various realigned Zephyr-7b models $\pirealignedapprox$ by linearly adjusting $\lambda$ from $0$ to $2.0$ in increments of $0.2$. We then evaluate these models using MT-Bench \citep{Zheng2023judging}, demonstrating that $\lambda$ effectively controls the alignment level (Figure~\ref{fig:zephyr-realigned}).}

\begin{figure*}[ht]
    \centering
  \includegraphics[width=0.98\textwidth]{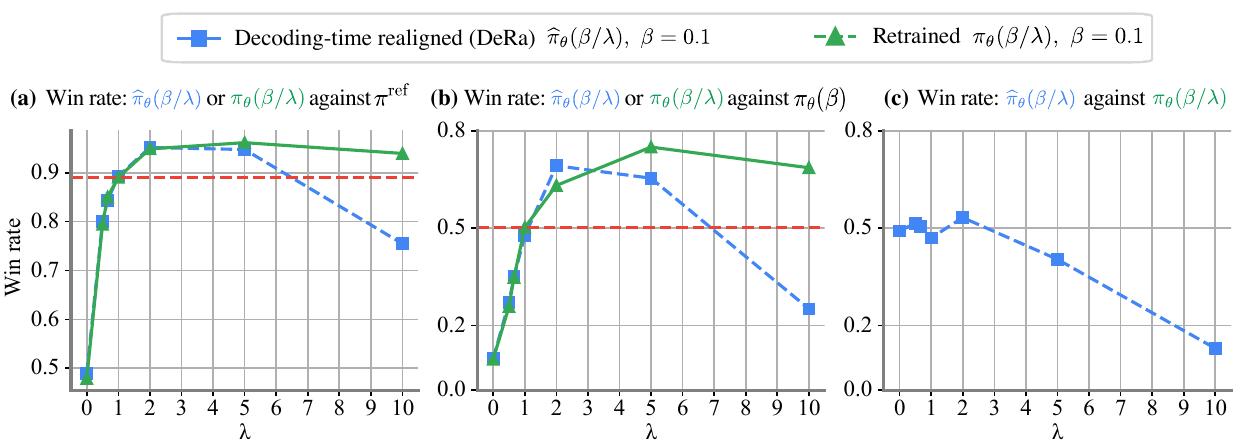}
  \vspace{-0.5cm}
  \caption{\textbf{Comparing \gls{dera} and retrained models with different KL strengths in the summarization task}. Model $\pi_{\theta}$ are trained with the policy gradient method (see Appendix~\ref{appendix:summarization}).
  Panel (a): comparing \gls{dera} models ($\squaremarker$) or retrained model ($\triangleupmarker$) against the reference model. Panel (b): comparing \gls{dera} ($\squaremarker$) or retrained model ($\triangleupmarker$) against the base-aligned model. Panel (c): comparing \gls{dera} against the retrained model.
  These results demonstrate that: (i) the performance of \gls{dera} and retrained model is closely related, and (ii) \gls{dera} enables the identification of KL strengths $\beta/\lambda$ that outperform the original base KL strength $\beta$, whose performance is indicated by the red lines.
  \label{fig:summarization-realigned}}
\end{figure*}

\subsection{Learning to summarize from human feedback \label{subsec:learning-to-summarize}}

We now tackle the more complex learning-to-summarize task using the
Reddit TL;DR summarization dataset from \citet{Stiennon2020learning}.
Our goal is to empirically test whether \gls{dera}
can effectively guide the search for a suitable KL regularizing strength.

\paragraph{Obtaining SFT and aligned models.} 

We first train a SFT model $\piref$ based on a pre-trained T5-Large model \cite{Raffel2020exploring}, following the procedure in \citet{Stiennon2020learning} and \citet{Munos2023nash}.
Next, we train a
separate T5-Large model to serve as a reward model using the preference dataset from
\citet{Stiennon2020learning}. Finally, we perform alignment training to optimize
the expected reward, regularized with a KL strength of $\beta=0.1$, using a
policy gradient approach. This yields an aligned model $\pi_\theta(\beta)$ with $\beta=0.1$. 
Experimental details are provided in
Appendix~\ref{appendix:summarization}.

\paragraph{Apply DeRa.}
To investigate whether alternative KL strengths $\beta/\lambda$ could outperform
the base KL strength $\beta$, we apply \gls{dera} 
by varying $\lambda$ over a wide
range of values $\{0.5, 2/3, 1.0, 2.0, 5.0, 10.0\}$, which corresponds to KL
strengths of $\beta/\lambda$ in the range $\{0.2, 0.15, 0.1, 0.05, 0.02,
0.01\}$. Note
that the cases of $\lambda=2.0,5.0$, or $10.0$ involve extrapolation, 
where the combined logits lie outside the convex hull of the reference and aligned logits.
Our aim is to stress test these extrapolating $\lambda$ values to assess whether
\gls{dera} can still provide a reasonable approximation when using these values.

\paragraph{Evaluating models.}

To evaluate the performance of \gls{dera} models $\pirealignedapprox$ at
different KL strengths $\beta/\lambda$, we use the highly capable PaLM 2 Large
model \citep{Anil2023palm} as a judge\footnote{Specifically, we use the {\tt text-unicorn-001} version.}. We extract 1000 queries from the
summarization dataset's test fold and generate responses to these queries using
each realigned model $\pirealignedapprox$. We then have the PaLM 2 model
identify the better response from each pair of responses sampled from two
different models. The win rate of each model against its counterpart is
calculated by dividing the total number of samples preferred by the first model
over the second model. These win rates are presented in
Figure~\ref{fig:summarization-realigned}\textbf{(a)}, where the win rates of DeRa
against the reference are indicated by $\squaremarker$. Additionally, the
win rates of the retrained-from-scratch models $\pi_{\theta}(\beta/\lambda)$
against the SFT model $\piref$, indicated by $\triangleupmarker$, are evaluated
as a benchmark.
Similarly, Figure~\ref{fig:summarization-realigned}\textbf{(b)} compares the win
rates of each model against the aligned model $\pi_{\theta}(\beta)$. Finally,
Figure~\ref{fig:summarization-realigned}\textbf{(c)} showcases the win rates of
sampling from DeRa $\pirealignedapprox$ against sampling from the retrained-from-scratch models
$\pi_{\theta}(\beta/\lambda)$.

\paragraph{Findings.} 

Figure~\ref{fig:summarization-realigned} shows two key results. First,
\gls{dera} effectively identifies KL strengths $\beta/\lambda$ that outperform
the default KL strength $\beta$. 
These values are represented by the win rates above the
red lines. Second, \gls{dera} and retrained models generally agree well.
Notably, this high agreement persists even with $\lambda > 1$, which are values
in the extrapolation regime.
This high agreement suggests that we can use \gls{dera}
as a cost-effective yet accurate method for
determining effective KL strength hyperparameter.

\paragraph{Under- and over-regularization.} 

In Figure~\ref{fig:summarization-realigned}, \gls{dera} suggests that the aligned model $\pi_{\theta}(\beta)$ might be over-regularized. Reducing the KL regularization to $\beta/\lambda$ for $\lambda>1$ enhances the win rate, as confirmed by retrained models. In Figure~\ref{fig:dpo-realigned-sax-comparison} in the Appendix, we show that \gls{dera} is also capable of identifying under-regularized models. In that case, \gls{dera} recommends a greater KL strength (with a $\lambda<1$) for improved performance.

\subsection{Hallucination mitigation}
To illustrate \gls{dera} on another real-world problem, we {provide a qualitative example of hallucination mitigation} for Retrieval Augmented Generation (RAG; \citealp{Lewis2020retrieval}), popularized by the recent application of LMs to search engines such as Bing and Google. Specifically, our task is to rewrite a given list of pro and con arguments in natural prose. Importantly, the rewritten text must strictly adhere to the semantics of the original arguments, without introducing new content, that is, without hallucinating.

\paragraph{Experimental setup.} We first train the PaLM 2 model \citep{Anil2023palm} to act as a reward model using LoRA \citep{Hu2022lora}, then use a LoRA-based RLHF \citep{Sun2023exploring} regularized with a KL of strength \(\beta = 0.1\) to align and mitigate hallucinations on a second PaLM 2\footnote{Specifically, for this we use the {\tt text-bison-001} version.}. We then apply \gls{dera} by varying $\lambda$ over $\{0.011, 0.1,  0.5, 0.67, 1, 2, 5, 10\}$, which corresponds to regularization strengths $\beta/\lambda$ in the range $\{ 9,   1,   0.2,   0.15, 0.1,   0.05,  0.02, 0.01\}$. The details of the datasets used for the reward model and RLHF are presented in Appendix~\ref{appendix:hallucination-mitigation}.

\begin{figure}[ht]
    \centering
  \includegraphics[width=0.47\textwidth]{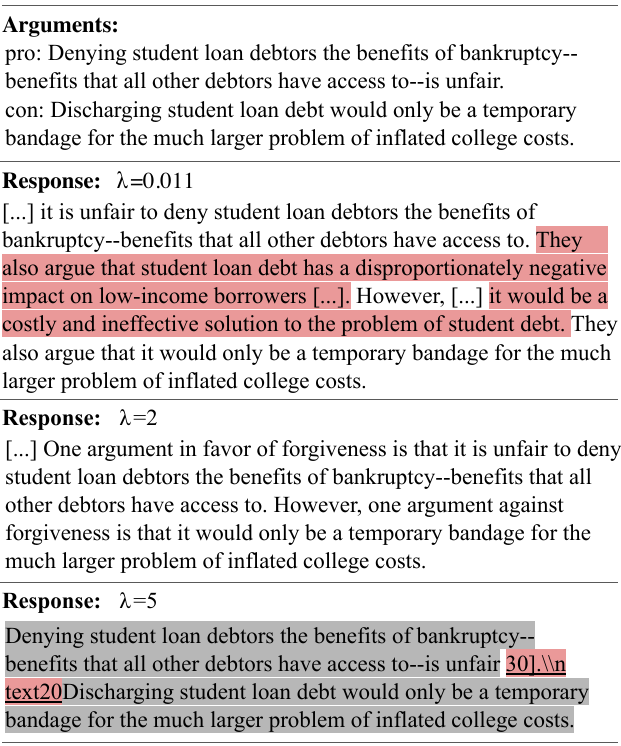}
  \caption{{\bf DeRa can control hallucinations in neutral response generation.} With a small $\lambda$ (limited alignment), the sampled response includes hallucinations (\colorbox{highlightRed}{highlighted in red}), meaning semantic content not present in the argument provided. Increasing $\lambda$ to $2$ reduces hallucinations. However, at excessively high $\lambda$, the model begins to copy the argument verbatim (\colorbox{highlightGray}{highlighted in gray}), indicating reward hacking, and produces incoherent responses (\colorbox{highlightRed}{\underline{highlighted in red and underlined}}).}
  \label{fig:hallucination-rate}
\end{figure}

\paragraph{Findings.} The results are presented in Figure \ref{fig:hallucination-rate}. Similarly to what was observed in Figure~\ref{fig:dera-header}, we see that changes in \(\lambda\) lead to changes in the style of generation. More precisely, low values of \(\lambda\) lead to a behavior more similar to the reference model, and thus a higher tendency to hallucinate. As we increase \(\lambda\), we improve the desired behavior of the model, where the arguments are correctly rewritten in natural prose without hallucinations. However, if when \(\lambda\)  increases further, the naturalness of the language decreases, with the model resorting to copying and pasting the initial arguments verbatim, and losing coherence.

\section{Related work}

\paragraph{Multi-reward RLHF.}

Several studies have explored the issue of balancing multiple rewards
\citep{Rame2023rewarded, Jang2023personalized, Mitchell2024emulator}.  Unlike
our approach, which combines a reference model and an aligned model,
multi-reward methods aim to combine multiple models obtained from different
rewards, each independently trained from the reference model. These methods are
driven by the recognition that humans have diverse expectations when interacting
with language models, each characterized by a reward function.  For instance, in
training chatbots to serve as human assistants, two pertinent reward functions
are the chatbot's helpfulness and harmlessness. After training, these models are
combined and weighted, either through parameter interpolation
\citep{Rame2023rewarded, Jang2023personalized} or model ensemble
\citep{Mitchell2024emulator}, enabling to control the strength of each reward.

\paragraph{Proxy approaches for finetuning.} 

Emulated fine-tuning (\citealt{Mitchell2024emulator}; EFT) explores the capabilities of language
models across two dimensions: model scale (large vs small models) and training stages (pretraining vs finetuning). 
EFT is a scale-decoupling approach that transfers the finetuning effects of a small LM to a large LM and vice versa. While \citet{Mitchell2024emulator} used this approach as a tool for analyzing the capabilities of LMs, \citet{Liu2024tuning} demonstrated the empirical effectiveness of this proxy-tuning approach, showing it competes with standard finetuning across various benchmarks.
Our approach shares similarities with EFT/proxy-tuning in that both merge trained models at the output level. While this approach decouples model scales, our
objective is to strike a balance between reward and regularization.
To that end, we show that our approach is a cost-effective proxy for full retraining with different degrees of alignment. {Furthermore, \citet{Lu2023inference, Deng2023reward, Khanov2024alignment, Huang2024deal} have explored using auxiliary models, such as separate reward models, to guide the text generation process of LMs.}

\paragraph{Other decoding approaches that merge logits.}
Several sampling approaches merge logits from multiple language models like DeRa, but with different objectives. The fusion approaches \citep{Gulcehre2015using, Stahlberg2018simple} aim to use monolingual LMs to improve machine translation. Contrastive decoding approaches \citep{Li2023contrastive} aim to enhance incoherence and lexical diversity in text generation. The Classifier-free guidance approach aims to improve prompt adherence in text generation. Speculative decoding approaches \citep{Leviathan2023fast, Chen2023accelerating} speed up token generation by outputting multiple tokens at a time. 

\section{Conclusion}
We introduced DeRa, a method for adjusting the KL regularization strength during decoding for realigning language models. Based on the variational perspective of the KL-regularized alignment objective in \eqref{eq:kl-regularized-aligned-model}, 
we proved that aligned models with varying KL regularization strengths are all geometric mixtures; these mixtures combine the probability distributions of a reference model and an aligned model, varying only in their mixing weights.
DeRa uses this knowledge to approximate these mixtures autoregressively during decoding. This approach gives DeRa a simple implementation and a clear interpretation. 

One of DeRa's advantages is its ability to adjust regularization levels for individual users, prompts, or tasks. Many open-weights models, such as Llama-2 \citep{Touvron2023llama2} and Mistral 7B \citep{Jiang2023mistral}, offer publicly available checkpoints for base and instruction-fine-tuned models. These models can act as references and aligned models alongside DeRa, allowing practitioners to tailor language model alignment to specific user preferences or downstream applications. Additionally, as we experimentally validated, DeRa can be used to efficiently identify promising regularization strengths to retrain a model. This streamlines hyperparameter tuning and reduces computational costs by avoiding unnecessary retraining across a wide range of regularization strengths.

\section*{Impact statement}
We propose a method for exploring and adjusting regularization strengths in language model alignment. 
This work should be viewed within the broader context of language model alignment techniques, aiming to promote friendly, safe, and harmless responses in language models.
It can also be viewed as a way to sweep over the regularization strength at decoding time, streamlining hyperparameter selection and reducing the number of retraining runs to get optimal regularization strengths.

\section*{Acknowledgment}
We thank Johan Ferret for his helpful feedback on a draft of this paper. 
We thank Bilal Piot and Rémi Munos for their help with the auto-evaluator for the summarization task.

\bibliography{project}
\bibliographystyle{icml2024}

\newpage
\appendix
\onecolumn

\section{Proof of Proposition~\ref{prop:logits-combine} \label{appendix:logits-proof}}

\logitscombine*

\begin{proof}[Proof of Proposition~\ref{prop:logits-combine}]

We denote $\vp_t^\textrm{ref} \coloneqq \piref(\cdot|x, y_{1:t})$ and $\vp_t^{\theta} \coloneqq \pi_\theta(\beta)(\cdot|x, y_{1:t}).$ We note that

\begin{align}
\vp_t^\textrm{ref} \odot \big[\vp_t^{\theta} \oslash \vp_t^\textrm{ref} \big]^\lambda & =\begin{bmatrix}
           \frac{\exp(\vhref_t[1])}{\sum_{i=1}^V\exp(\vhref_t[i])} \\
           \vdots \\
          \frac{\exp(\vhref_t[V])}{\sum_{i=1}^V\exp(\vhref_t[i])} 
\end{bmatrix} \odot \begin{bmatrix}
           \frac{\exp(\vh_t^{\theta}[1])}{\sum_{i=1}^V\exp(\vh_t^{\theta}[i])} \\
           \vdots \\
          \frac{\exp(\vh_t^{\theta}[V])}{\sum_{i=1}^V\exp(\vh_t^{\theta}[i])} 
\end{bmatrix}^\lambda
\oslash \begin{bmatrix}
           \frac{\exp(\vhref_t[1])}{\sum_{i=1}^V\exp(\vhref_t[i])} \\
           \vdots \\
          \frac{\exp(\vhref_t[V])}{\sum_{i=1}^V\exp(\vhref_t[i])} 
\end{bmatrix}^\lambda \\
& = \frac{1}{\Big ( \sum_{i=1}^V\exp(\vhref_t[i]) \Big)^{1-\lambda} \Big (\sum_{i=1}^V\exp(\vh_t^{\theta}[i]) \Big)^{\lambda}} \exp \Big (\lambda \vh_t^{\theta} + (1-\lambda)\vhref_t \Big),
\end{align}
where $\odot$ and $\oslash$ are entrywise product and division. It follows that 
\[ \widehat{\pi}_{\theta
}(\cdot|x, y_{1:t}) = \frac{ \vp_t^\textrm{ref} \odot \Big [ \vp_t^{\theta} \oslash \vp_t^\textrm{ref} \Big ]^\lambda}{\Big \|\vp_t^\textrm{ref} \odot \Big [ \vp_t^{\theta} \oslash \vp_t^\textrm{ref} \Big ]^\lambda \Big\|_1} = \frac{ \exp \Big ((1-\lambda)\vhref_t + \lambda \vh_t^{\theta} \Big)}{\Big\| \exp \Big ((1-\lambda)\vhref_t + \lambda \vh_t^{\theta} \Big)\Big\|_1} = \softmax \Big(\lambda \vh_t^{\theta} + (1-\lambda)\vhref_t \Big)\]
as claimed.
\end{proof}

\section{Linear combination of multiple rewards}
\label{app:multiple-rewards}
Following the notation of Section~\ref{sec:background}, we consider the case of a linear combination of rewards $r_\lambda$ defined by
\[
r_\lambda(x, y) = \sum_{i=1}^K \lambda_i r_i(x, y)\, ,
\]
for $K$ reward functions $r_i$ and $\lambda = (\lambda_1, \ldots, \lambda_K) \in \R^K$.

Analogously writing $\pirealignedm$ for the realigned optimal distribution under the reward $r_\lambda$, we have that 
\begin{align*}
\pirealignedm(y | x) &= \frac{\piref(y | x) \exp[\frac{1}{\beta} r_\lambda(y |x)]}{\sum_{y'} \piref(y' | x) \exp[\frac{1}{\beta} r_\lambda(y' |x)]}\\
&= \frac{\piref(y | x) \exp[\frac{1}{\beta}  \sum_{i=1}^K \lambda_i r_i(x, y)]}{\sum_{y'} \piref(y' | x) \exp[\frac{1}{\beta}  \sum_{i=1}^K \lambda_i r_i(x, y')]}\, .
\end{align*}
Denoting by $\pialignedi$ the optimal distribution realigned under the reward $r_i$, we have that
\[
\frac{1}{\beta} r_i(x, y) = \log \frac{\pialignedi(y | x)}{\piref(y, x)} + \log Z_i(x)\, . 
\]
Plugging this into the realigned distribution expression yields
\begin{align*}
\pirealignedm(y | x) &= \frac{\piref(y | x) \prod_{i=1}^K \Big(\frac{\pialignedi(y | x)}{\piref(y, x)} Z_i(x) \Big)^{\lambda_i}}{\sum_{y'} \piref(y' | x) \prod_{i=1}^K \Big(\frac{\pialignedi(y' | x)}{\piref(y', x)} Z_i(x) \Big)^{\lambda_i}}\\
&= \frac{\piref(y | x)^{1-\bar \lambda} \prod_{i=1}^K \pialignedi(y | x)^{\lambda_i}}{\sum_{y'} \piref(y' | x)^{1-\bar \lambda} \prod_{i=1}^K \pialignedi(y' | x)^{\lambda_i}}\, ,
\end{align*}
where $\bar \lambda = \mathbf{1}^\top \lambda$ is the sum of the $\lambda_i$. Use of the logits $h_t^{\theta_i}(\beta)$ with multiple linear combinations (and autoregressive approximation) carries through in a similar fashion.

\section{Alignment methods \label{appendix:alignment-methods}}

\paragraph{Policy gradient methods}

The policy gradient method updates the \gls{lm} with an estimate of the following gradient:

\begin{equation} \label{eq:vanilla-policy-gradient}
\EE_{x \sim \rho, y \sim \pi_{\theta}(\cdot|x)} \left[ \nabla_{\theta} \log \pi_{\theta} (y|x) \Big ( R(x,y) - \beta \textrm{KL} \big(\pi_{\theta}(\cdot | x), \piref(\cdot | x) \big)  \Big)  \right].
\end{equation}

The proximal policy optimization (PPO; \citealt{Schulman2017proximal}) is a variant of the vanilla policy gradient optimization. It replaces reward with
general advantage estimation \citep{Schulman2015high} and introduces clipped probability ratios. For the length reward problem, we use \gls{ppo}. For the summarization task, we used the vanilla policy gradient optimization \eqref{eq:vanilla-policy-gradient}, consistent with \citet{Munos2023nash}.

\paragraph{Direct preference optimization}
Direct preference optimization (DPO; \citealp{Rafailov2023direct}) is an approach that directly optimizes the policy through a loss function defined via the Bradley-Terry reward model, without using a reward function. Given a dataset $\mathcal{D}$ that contains tuples of a query and two responses favored and unfavored by humans $(x, y_w, y_l)$, the \gls{dpo} loss is defined as
$$
\mathcal{L}_{\mathrm{DPO}}\left(\pi_\theta ; \piref\right)=-\mathbb{E}_{\left(x, y_w, y_l\right) \sim \mathcal{D}}\left[\log \sigma\left(\beta \log \frac{\pi_\theta\left(y_w \mid x\right)}{\piref\left(y_w \mid x\right)}-\beta \log \frac{\pi_\theta\left(y_l \mid x\right)}{\piref\left(y_l \mid x\right)}\right)\right].
$$
We use DPO in the summarization task (Appendix~\ref{appendix:summarization}) and in the chat model alignment task (Appendix~\ref{appendix:chat-models}).

\section{Details of experiment setup \label{appendix:experimental-details}}

\subsection{Length-reward experiments \label{appendix:length-reward}}
For this toy task of length reward, we use T5-small \citep{Raffel2020exploring} for policy and reward models. In supervised finetuning, we take a pretrained T5-small model \citep{Roberts2022t5x} and fine tune it on the Xsum dataset \citep{Narayan2018dont} with 15k steps in a batch size of $32$, yielding a model $\piref$. With $\piref$ as an initialization, we train aligned policy models $\pi^\ast(\beta/\lambda)$ to maximize the length reward \eqref{eq:length-reward} using PPO. The policy learning rate is 5e-6, and the value function learning rate is 1e-5.

As shown in the main text, a lower KL regularization $\beta/\lambda$ (i.e., with a greater $\lambda$), allows aligned models $\pi^\ast(\beta/\lambda))$ to gain more length reward (Figure~\ref{fig:length-reward-result}). However, as shown in Figure~\ref{fig:sax-compare-length-reward-models}, a greater KL strength (a smaller $\lambda$) retains a higher summarization quality. We measured the summarization quality in a way identical to our approach in Section~\ref{subsec:learning-to-summarize}: we prompt a highly capable Palm 2 model and ask it to compare the win rate of the aligned model against the SFT model; see Appendix~\ref{appendix:summarization} for more details of the auto-evaluation setup.

\begin{figure}[!ht]
    \centering
  \includegraphics[width=0.45\textwidth]{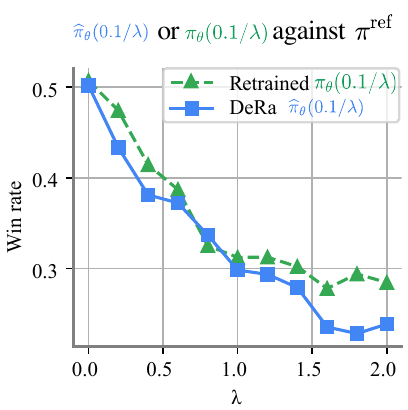}
  \caption{Comparing \gls{dera} and retrained models against the SFT model using Palm 2 auto-evaluation; the win rate of both \gls{dera} and retrained model decrease, since the length reward is an ineffective proxy of summarization quality. \label{fig:sax-compare-length-reward-models}}
  \vspace{-0.5cm}
\end{figure}

\subsection{Controlling alignment: a qualitative demonstration \label{appendix:qualitative-demo}}
For this experiment, we use the checkpoints of SFT and aligned Zephyr-7b models \citep{Tunstall2023zephyr}. To save memory, we load both SFT and aligned checkpoints\footnote{\url{https://huggingface.co/HuggingFaceH4/zephyr-7b-beta}} with 4-bit quantization. We follow the Colab demo of 
Zephyr\footnote{\url{https://huggingface.co/HuggingFaceH4/zephyr-7b-alpha/blob/main/colab-demo.ipynb}}, and set \texttt{max\_new\_tokens=256}, \texttt{temperature=0.7}, \texttt{top\_k=50}, and \texttt{top\_p=0.95} for sampling.

\subsection{Summarization experiments \label{appendix:summarization}}

\paragraph{Model training.} 
For this experiment, we use a T5-Large model as the policy model, and a separate T5-Large model as the reward model, as in \citet{Munos2023nash}. The supervised finetuning and policy gradient alignment setting mirrors the settings of \citet{Munos2023nash}. For DPO alignment, we use $\beta=0.1$, learning rate 3e-6, batch size 64, and 20k training steps.

\paragraph{Palm 2 evaluation.} 
We use Palm 2 (the version called {\tt text-unicorn-001}) to compare the quality of summarized responses. Given a to-be-summarized query $\texttt{⟨text⟩}$ and two summaries $\texttt{⟨summary1⟩}$ and $\texttt{⟨summary2⟩}$ sampled from two models $\pi_1$ and $\pi_2$, we prompt the Palm 2 Large \gls{lm} with:

\texttt{`You are an expert summary rater. Given a piece of text and two of its possible summaries, output 1 or 2 to indicate which summary is better.
Text - ⟨text⟩, Summary 1 - ⟨summary1⟩, Summary 2 - ⟨summary2⟩. Preferred Summary -'}

To avoid positional bias, we swap $\texttt{⟨summary1⟩}$ and $\texttt{⟨summary2⟩}$  with probability $0.5$, and then consistently swap back the generated preferences of Palm 2. We then compute the ratio of the number of times that responses from the $\pi_1$ are preferred over than the second model $\pi_2$; we call this ratio the win rate of $\pi_1$ against $\pi_2$. 

\paragraph{Pairwise acurracy evaluation.} 
In addition to using the Palm 2 for evaluation, we also consider alternative evaluation metrics. For a given \gls{lm} $\pi$, and a given pairwise dataset $D$ be a dataset that contains tuples of a query and two responses favored and unfavored by humans ($(x, y_w, y_l)$), we let the pairwise accuracy be

\begin{equation} \label{eq:pairwise-acc}
\text{PairAcc}(\pi; D) = \EE_{(x, y_w, y_l) \sim D}~\mathbf{1} \Big( \frac{1}{|y_w|} \log \pi(y_w | x) > \frac{1}{|y_l|} \log \pi(y_l | x)  \Big).
\end{equation}

Intuitively, the pairwise accuracy \eqref{eq:pairwise-acc} is high if the average log probability for the winning responses $y_w$ is generally greater than that of the losing responses $y_l$. Figure~\ref{fig:pairwise-accuracy} shows the pairwise accuracy of retrained models and \gls{dera} models at various $\lambda$. The results of \gls{dera} ($\squaremarker$) and retrained models ($\triangleupmarker$) are overall close for $\lambda$ being small; suggesting that \gls{dera} is a sensible proxy for the retrained model. However $\lambda$ increases, the gap between the performance of \gls{dera} and the retrained model increases. This is expected, since $\lambda > 1$ is in the extrapolation region, where $\pirealignedapprox$ in \eqref{eq:adapted-kl-logits-softargmax} fails to be a good approximator of $\pi^\ast(\beta/\lambda)$.

\paragraph{}
\begin{figure}[!ht]
    \centering
  \includegraphics[width=0.45\textwidth]{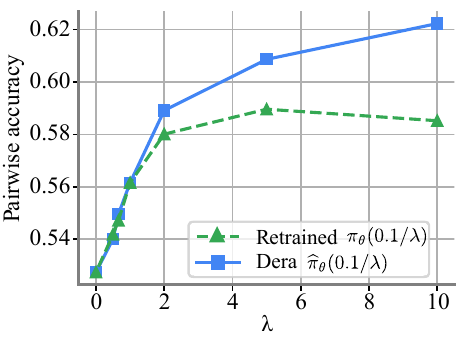}
  \vspace{-0.cm}
  \caption{Comparing \gls{dera} $\widehat{\pi}(0.1/\lambda)$ and retrained models $\pi_{\theta}(0.1/\lambda)$ with different KL strengths in the summarization task, using the pairwise accuracy metric. \label{fig:pairwise-accuracy}}
\end{figure}

\paragraph{DPO alignment.} While we only reported policy gradient alignment result in the main text, here we report result from DPO alignment (Figure~\ref{fig:dpo-realigned-sax-comparison}).

\begin{figure}[!ht]
    \centering
  \includegraphics[width=0.99\textwidth]{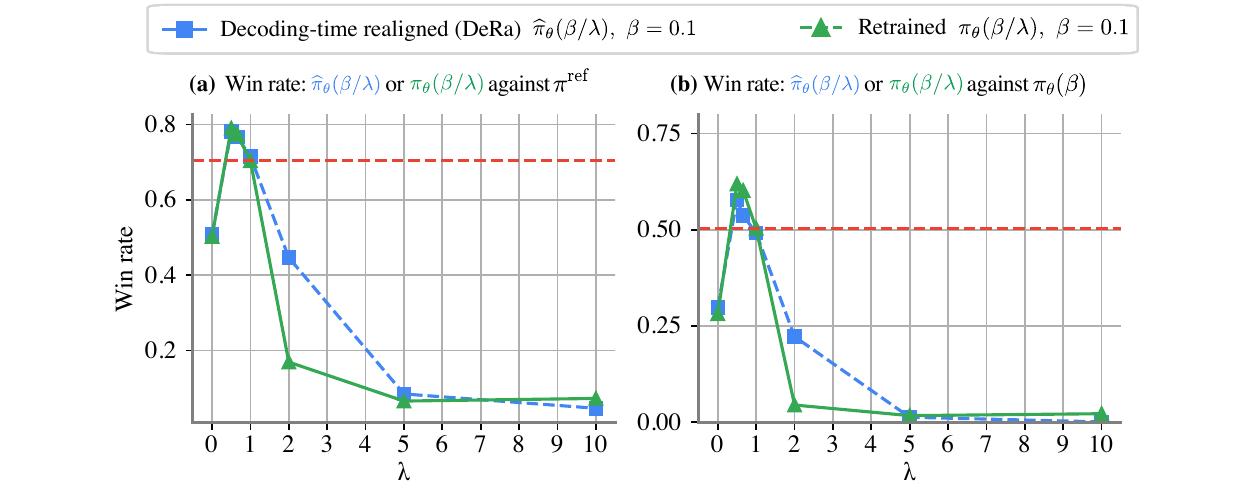}
  \caption{\textbf{Comparing \gls{dera} $\widehat{\pi}(\beta/\lambda)$ and retrained models $\pi_{\theta}(\beta/\lambda)$ with different KL strengths in the summarization task}. All models $\pi_{\theta}$ are optimized with DPO. Panel (a): compare \gls{dera} models $\widehat{\pi}_{\theta}(\beta/\lambda)$ ($\squaremarker$) or retrained model $\pi_{\theta}(\beta/\lambda)$ ($\triangleupmarker$) against the reference model. Panel (b): compare \gls{dera} $\widehat{\pi}_{\theta}(\beta/\lambda)$ ($\squaremarker$) or retrained model $\pi_{\theta}(\beta/\lambda)$ ($\triangleupmarker$) against the base-aligned model $\pi_{\theta}(\beta)$.
  These comparisons show that performances of \gls{dera} and the retrained model are correlated.  \label{fig:dpo-realigned-sax-comparison}}
\end{figure}

\subsection{Hallucination mitigation} \label{appendix:hallucination-mitigation}
\paragraph{Training the reward model.} The dataset used to train the reward model contains 1888 examples, which are split into training, validation, and evaluation datasets with 723, 242, and 223 examples respectively. Out of the examples in the training split, 388 contain hallucinations. Each example is a triple consisting of (i) a prompt with instructions to rewrite a given list of arguments in natural language (ii) the corresponding generation by the model (iii) a human annotated hallucination score (1 if the generation does not contain hallucinations and 0 if it does). The quality of the human annotations was checked by a paid third-party pool of raters. Our final reward model acting as a classifier achieves an ROC-AUC of 0.986 on the evaluation set.

\paragraph{Details on RLHF.} The dataset used to perform the RLHF contains 91 examples, which are split into training, validation, and evaluation datasets with 60, 20, and 11 examples respectively. This seemingly small dataset is compatible with the parameter-efficient tuning technique used (LoRA). Each example in these datasets consist of an user query on a sensitive question accompanied by a set of arguments in favor and against the topic in question, along with an expected “good” answer (without hallucinations). The temperature used for generations is set to \(T = 1\) and we run RLHF for a thousand steps.

\paragraph{Prompt used for evaluation.} To evaluate the resulting quality of the model after RLHF, we used a state of the art LLM as an auto-rater for our task. Our prompt uses two techniques, namely we provide six few-shot examples in it and we demand the model to execute the task as an expert. Each few-shot example has the following structure: 

\texttt{Reference arguments: [ \{arguments\} ]\\
The natural language style paragraph version of these arguments: [ \{paragraph\} ]\\
Expert review: the natural language paragraph contains additional points to the reference arguments (yes/no):}

The evaluation prompt is then constructed using the following template:

\texttt{The following are examples of an expert noting when a natural language style paragraph of text contains additional arguments to a given set of reference arguments on a topic.\\
\\
\{fewshot example 1\}\\
\\
\{fewshot example 2\}\\
\\
\{fewshot example 3\}\\
\\
\{fewshot example 4\}\\
\\
\{fewshot example 5\}\\
\\
\{fewshot example 6\}\\
\\
Expert review of an additional case where it was not initially known if it contained extra arguments or not:\\
\\
Reference arguments: [ \{arguments\} ]\\
The natural language style paragraph version of these arguments: [ \{paragraph\} ]\\
Expert review: the natural language paragraph contains additional points to the reference arguments (yes/no):}

\subsection{Aligning general-purpose chat models \label{appendix:chat-models}}
We apply \gls{dera} to adjust the KL regularization strength for general-purpose chat models. We focus on Zephyr-7b \citep{Tunstall2023zephyr}, a high-performing, open-weight chat model that is fine-tuned based on the Mistral 7b model \citep{Jiang2023mistral} with \gls{dpo}. The open-weight checkpoints of the Zephyr-7b contain both the SFT model $\piref$, and the aligned model $\pi_{\theta}(\beta)$ trained at the KL strength $\beta = 0.1$ \citep{Tunstall2023zephyr, Tunstall2023alignment}. With DeRa, we can flexibly explore the performance of Zephyr-7b models at different KL strengths.

We apply \gls{dera} to obtain different realigned Zephyr-7b models $\pirealignedapprox$ by linearly sweeping $\lambda$ from $0$ to $2.0$ with a stepsize $0.2$. We then validate these models using MT-Bench \citep{Zheng2023judging}, a dataset with 160 questions across eight knowledge areas, formatted in a multi-turn style. The performance of different realigned models $\pirealignedapprox$, indicated by $\squaremarker$ in Figure~\ref{fig:zephyr-realigned}, is evaluated based on the quality of its responses to these questions, with GPT-4 providing scores ranging from 1 to 10. Inspired by the rewarded soup \citep{Rame2023rewarded} approach, we also evaluated a weight-combine variant of DeRa, whose performance is indicated by $\circlemarker$ in Figure~\ref{fig:zephyr-realigned}. While the standard DeRa ($\squaremarker$) linearly combine logits (Algorithm~\ref{alg:dera}), the weight-combining DeRa linearly combine parameters of the aligned model and the reference model, as in \citet{Rame2023rewarded}.

\begin{figure}[ht]
    \centering
  \includegraphics[width=0.49\textwidth]{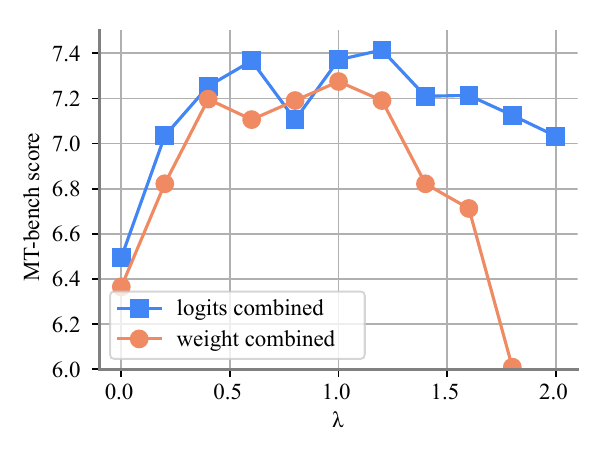}
  \vspace{-0.5cm}
  \caption{\textbf{Evaluating Zephyr-7b models with different DeRa-adjusted KL strengths on MT-bench}. These DeRa models $\pirealignedapprox$ use $\beta = 0.1$ and $\lambda \in \{0, 0.2, 0.4, \ldots, 2.0\}$ on MT-bench. The cases of $\lambda=0$ and $\lambda=1$ match the SFT model $\piref$ and aligned checkpoints $\pi_{\theta}(\beta)$ provided in the Zephyr-7b release \citep{Tunstall2023zephyr}. \label{fig:zephyr-realigned}}
\end{figure}

Based on Figure~\ref{fig:zephyr-realigned}, we anticipate a limited potential for improvement by adjusting KL regularization strength. This is because the base KL strength with $\lambda=1$ pretty much attains the best MT-bench score in Figure~\ref{fig:zephyr-realigned}. However, based on Figure~\ref{fig:zephyr-realigned}, we hypothesize that another sensible choice of KL strength is around $\lambda=0.5$. 
This is because the MT-bench score remains close to its peak at this value, and a lower $\lambda$ leads to stronger KL regularization, which can be beneficial for tasks that rely on the SFT model's capabilities.
One such task is mathematical reasoning, as alignment training generally diminishes models' performance on math reasoning problems like the GSM8k task \citep{Beeching2023open} of solving grade school math problems. To test this hypothesis, we re-train Zephyr-7b at $\lambda=0.5$ or $\beta/\lambda = 0.2$. We confirm that its MT-bench score closely matches the original Zephyr model's score at with a KL strength $0.1$. Additionally, we demonstrate that a stronger KL strength preserves the model's strong performance on math problems, leading to an overall higher score on the Open LLM leaderboard \citep{Beeching2023open}, as shown in Table~\ref{tab:detailed-open-llm} in the Appendix.

\begin{figure}[!ht]
    \centering
  \includegraphics[width=0.99\textwidth]{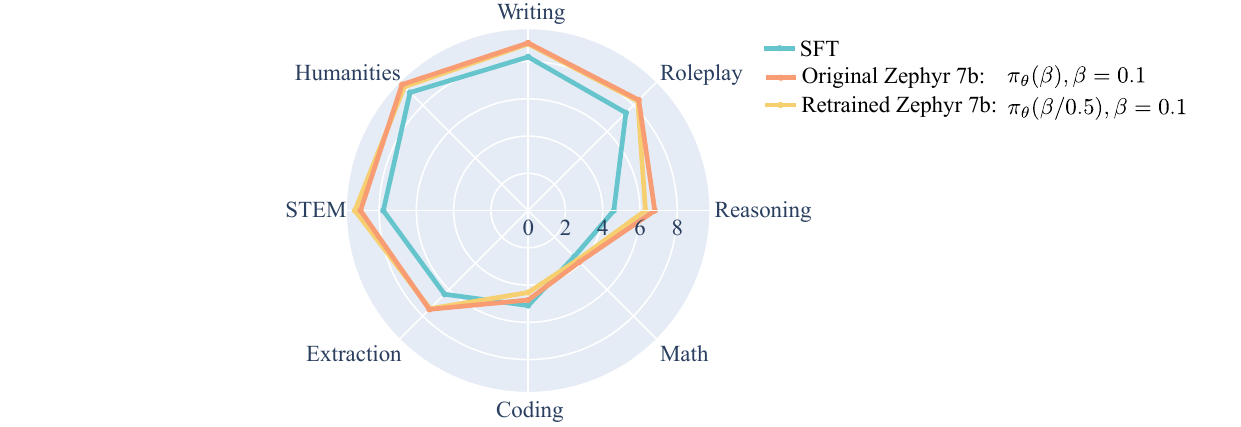}
\caption{Comparing MT-bench scores of the original Zephyr 7b \citep{Tunstall2023zephyr} with KL strength $0.1$ and the retrained Zephyr 7b at KL strength $0.1/0.5$. The results are very close. The average score of the original Zephyr 7b $\pi^\ast(\beta)$ with $\beta = 0.1$ is $7.37$; the average score of the retrained model $\pi^\ast(\beta/\lambda)$ with $\beta = 0.1$ and $\lambda=0.5$ is 7.23. \label{fig:mt-bench-retrained}}
\end{figure}

\begin{table*}[t]
\centering
\small
\begin{tabular}{lccccccc}
\toprule
& Average & ARC & HellaSwag & MMLU & TruthfulQA & Winogrande & GSM8K \\
\midrule
Zephyr-7b-beta-SFT       & 59.78  & 57.42 & 82.23 & 61.42 & 43.58 & 77.58 & \bf 36.47 \\
Zephyr-7b-beta [original; $\pi_{\theta}(0.1)$]       & 59.23  & \bf 62.03 & \bf 84.36 & 61.07 &\bf  57.45 & \bf 77.74 & 12.74 \\
Zephyr-7b-beta [retrained; $\pi_{\theta}(0.1/0.5)$] & \bf 61.55 & 61.77 & 84.04 & \bf 61.79 & 54.72 & 76.95 & 30.02\\
\bottomrule
\end{tabular}
\caption{Open LLM leaderboard. \label{tab:detailed-open-llm}}
\end{table*}

\end{document}

%% file: math_commands.tex
\usepackage{amsmath,amsfonts,bm}

\def\1{\bm{1}}

\def\vh{{\bm{h}}}

\def\vp{{\bm{p}}}

\DeclareMathAlphabet{\mathsfit}{\encodingdefault}{\sfdefault}{m}{sl}
\SetMathAlphabet{\mathsfit}{bold}{\encodingdefault}{\sfdefault}{bx}{n}

\def\Xcal{{\mathcal{X}}}
\def\Ycal{{\mathcal{Y}}}

\def\EE{{\mathbb{E}}}

\def\RR{{\mathbb{R}}}

\newcommand{\R}{\mathbb{R}}

\newcommand{\softmax}{\texttt{softmax}}

\DeclareMathOperator*{\argmax}{arg\,max}

\usepackage{hyperref} 
    \hypersetup{
    	colorlinks = true,
    	linkcolor = blue,
    	anchorcolor = blue,
    	citecolor = blue,
    	filecolor = blue,
    	urlcolor = blue
    }